\documentclass{article}

\usepackage[english]{babel}

\usepackage[letterpaper,top=2cm,bottom=2cm,left=3cm,right=3cm,marginparwidth=1.75cm]{geometry}

\usepackage{natbib}
\usepackage{amssymb}
\usepackage{amsmath,bm}
\usepackage{amsthm}
\usepackage{thmtools}
\usepackage{mathtools}
\usepackage{graphicx}
\usepackage[colorlinks=true, allcolors=blue]{hyperref}
\usepackage{unembedfig}

\usepackage{xspace}
\usepackage{thm-restate}

\theoremstyle{plain}
\newtheorem{theorem}{Theorem}[section]
\newtheorem{proposition}[theorem]{Proposition}

\theoremstyle{definition}
\newtheorem{definition}[theorem]{Definition}

\theoremstyle{remark}

\newcommand{\RR}{\mathbb{R}}

\newcommand{\vx}{\mathbf{x}}

\newcommand{\vf}{\mathbf{f}}
\newcommand{\vg}{\mathbf{g}}
\newcommand{\vl}{\mathbf{l}}
\newcommand{\vu}{\mathbf{u}}
\newcommand{\vv}{\mathbf{v}}
\newcommand{\vw}{\mathbf{w}}
\newcommand{\va}{\mathbf{a}}
\newcommand{\vb}{\mathbf{b}}
\newcommand{\vz}{\mathbf{z}}

\newcommand{\ie}{\textit{i.e.}\@\xspace}
\newcommand{\eg}{\textit{e.g.}\@\xspace}

\usepackage[capitalize]{cleveref}
\Crefname{equation}{Eq.}{Eqs.}
\Crefname{figure}{Fig.}{Figs.}
\Crefname{table}{Tab.}{Tabs.}
\Crefname{lemma}{Lemma}{Lemmas}
\Crefname{definition}{Def.}{Defs.}
\Crefname{appendix}{App.}{Apps.}
\Crefname{assumption}{Asm.}{Asms.}
\Crefname{section}{\S}{\S\S}
\crefformat{section}{\S#2#1#3} 
\Crefformat{section}{\S#2#1#3} 

\newcommand{\R}{\mathbb{R}}

\newcommand{\T}{^\top}

\newcommand{\snorm}[1]{\left\lVert#1\right\rVert_2}

\newcommand{\set}[1]{\left\{#1\right\}}
\newcommand{\cardinality}[1]{\left|#1\right|}

\newcommand{\sigmoid}[1]{\operatorname{sigmoid}\left(#1\right)}
\newcommand{\softmax}[1]{\operatorname{softmax}\left(#1\right)}

\newcommand{\sign}[1]{\operatorname{sign}{\left(#1\right)}}

\newcommand{\inputs}{\mathbf{x}}
\newcommand{\eye}[1]{\mathbf{I}_{#1}}
\newcommand{\fembed}{\vf(\vx)}
\newcommand{\funembed}{\mathbf{g}\left(y\right)}
\newcommand{\unembed}{\mathbf{W}}

\newcommand{\unembedb}{\mathbf{W'}}
\newcommand{\braid}{\mathbf{B}}

\newcommand{\unembedrow}[1]{\mathbf{w}_{#1}}
\newcommand{\unitrow}[1]{\mathbf{e}_{#1}\T}
\newcommand{\cossim}[2]{\operatorname{cos}\!\left(#1, #2\right)}
\newcommand{\cosmat}[1]{\mathbf{C}\!\left(#1\right)}
\newcommand{\signpattern}[1]{\mathcal{S}\left(#1\right)}
\newcommand{\rankingpattern}[1]{\mathcal{R}\left(#1\right)}

\newcommand{\skipgram}[2]{\operatorname{Skip-gram}\left(#1, #2\right)}

\newcommand{\colspan}[1]{\operatorname{colspan}\left(#1\right)}

\newcommand{\rank}[1]{\operatorname{rank}\left(#1\right)}

\newcommand{\expectation}[2][]{\mathbb{E}_{#1}\!\left[#2\right]}

\newcommand{\colcenter}[1]{\widetilde{#1}}

\newcommand{\zeroes}{\mathbf{0}}
\newcommand{\ones}{\mathbf{1}}

\title{What Cosine Similarity of Label Representations\\ Can and Cannot Tell us}
\author{Beatrix M. G. Nielsen \\ IT University of Copenhagen \\ \texttt{beat@itu.dk} 
\and Andreas Grivas \\ School of Mathematics, University of Edinburgh \\ \texttt{agrivas@ed.ac.uk}}
\date{}

\begin{document}
\maketitle

\begin{abstract}
Cosine similarity is often used to measure the similarity of vector representations of neural network models. However,  the cosine similarity of representations is not guaranteed to tell us anything about model probabilities. In this paper we show that for a softmax classifier, be it an image classifier or an autoregressive language model, the cosine similarity between label representations (called unembeddings in the paper) does not give any information on the probabilities assigned by the model. Specifically, we prove that given two unembeddings, it is possible to create another model which assigns the same probabilities for all inputs, but where the cosine similarity between the representations is now either $1$ or $-1$. We also show that for a sigmoid classifier (where each input can be assigned multiple labels), all pairwise cosine similarities between the unembeddings define the set of possible label combinations. However, for softmax classifiers (where each input is assigned a ranking of the labels from most to least likely), we need all pairwise cosine similarities between all \textit{differences} of unembeddings to know which rankings the model can predict. We conclude that it is misleading to interpret the cosine similarity between unembeddings without reference to the classifier that produced them.
\end{abstract}

\section{Introduction}

Suppose we have a language model (of the model class we introduce below \cref{def:softmax_classifier}) and we consider the vector representations for the two tokens ``dog'' and ``cat''. If we measure the cosine similarity between the two to be $0.95$, it is tempting to say that this is because ``dog'' and ``cat'' are tokens which co-occur in many similar contexts (reminiscent of Latent Semantic Analysis~\citep{deerwester1990} and Word2Vec~\citep{mikolov2013}), so the model will assign similar probabilities to these two tokens.
However, this reasoning is flawed for softmax classifiers:
it is possible to make a second model which matches the token probabilities of the original model for all inputs, but where the cosine similarity between ``dog'' and ``cat'' is now $-1$.
It is also possible to make a model which assigns very different probabilities to the ``dog'' and ``cat'' tokens, but where the cosine similarity of their representations is $1$.
This means that the cosine similarity between representations of tokens is disconnected from assigned probabilities and therefore should not be used for explaining those probabilities. 

While some studies that use cosine similarity on token representations are aware of its limitations, \eg ~\citet[Appendix A]{land2024}, and \citet{steck2024cosine} shows that cosine similarity can be arbitrary for certain linear models, there is to our knowledge no reference in the literature that highlights the problem for general softmax classifiers.
We fill this gap and further highlight that it is misleading to interpret the cosine similarity between unembeddings without reference to the classifier that produced them.
To show this, we characterise when the decision regions of softmax and sigmoid classifiers are determined by a matrix of cosine similarities between their unembeddings. 
%
In particular, if the unembeddings arise from a sigmoid classifier, all pairwise cosine similarities of the unembeddings determine its decision regions.
In contrast, for the unembeddings of a softmax classifier, we need all pairwise cosine similarities between all \textit{differences} of unembeddings to determine its fine-grained decision regions; pairwise cosine similarities of unembeddings are not enough.  

Our contributions are as follows:
\begin{itemize}
    \item For a softmax classifier (\cref{def:softmax_classifier}), we show that since probabilities are invariant to vector translations (\cref{lemma:Adding_vector_does_not_change_probs}) but cosine similarity is not (\cref{lemma:adding_vector_cosine_1,lemma:adding_vector_cosine_m1}), these two properties of the model are disconnected (\cref{theorem:equiv_models_diff_cosine}). Illustration in \cref{fig:three_models_same_probabilities_diff_cosines}.
    \item For a sigmoid classifier (\cref{eq:sigmoid_model}), we show that cosine similarities of unembeddings \textbf{can} tell us the possible label combinations (\cref{lemma:cos-sigmoid}).
    \item For a softmax classifier, we show that cosine similarities of unembeddings \textbf{cannot} tell us the possible rankings (\cref{lemma:counterexample}), for this we need cosine similarities between \textit{differences} of unembeddings (\cref{lemma:cos-softmax}).
\end{itemize}

\begin{figure*}[t!]    
    \centering
    \begin{minipage}{0.25\textwidth}
        \centering
        \includegraphics[width = \textwidth]{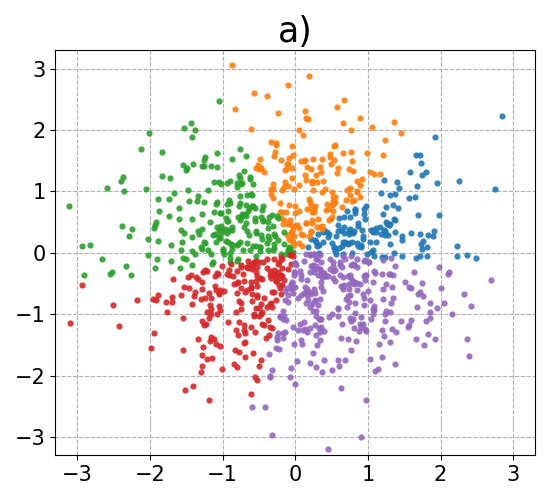}
    \end{minipage}%
    \begin{minipage}{0.25\textwidth}
        \centering
        \includegraphics[width = \textwidth]{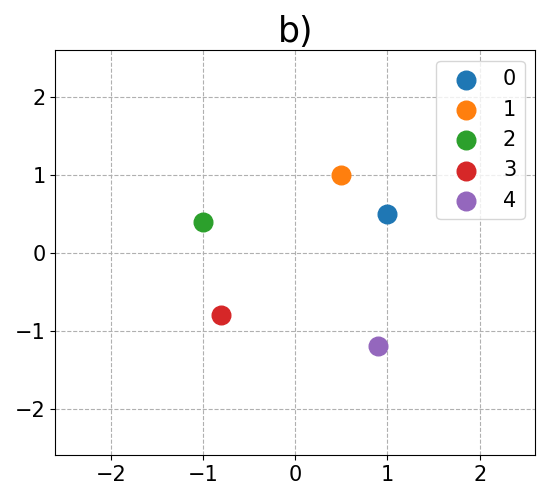}
    \end{minipage}%
    \begin{minipage}{0.25\textwidth}
        \centering
        \includegraphics[width = \textwidth]{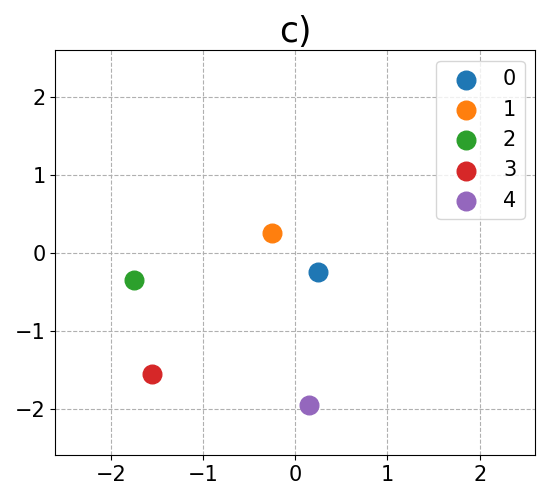}
    \end{minipage}%
    \begin{minipage}{0.25\textwidth}
        \centering
        \includegraphics[width = \textwidth]{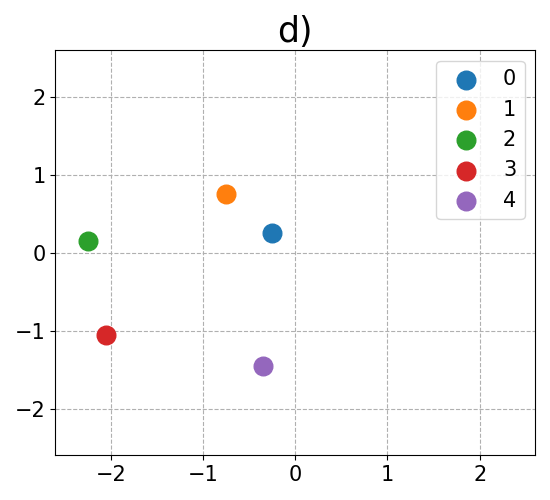}
    \end{minipage}%
    \caption{Example of three models which give same probabilities to all labels for all inputs, but where cosine similarities differ between the unembeddings, see the Appendix \cref{app:examples} for details. \textbf{a)} Embeddings coloured by highest probability label. These are the same for all three models. \textbf{b)} Unembeddings for model 1. Cosine between labels $0$ and $1$ is about $0.8$. \textbf{c)} Unembeddings for model 2. Cosine between labels $0$ and $1$ is $-1$. \textbf{d)} Unembeddings for model 3. Cosine between labels $0$ and $1$ is $1$. }
    \label{fig:three_models_same_probabilities_diff_cosines}
\end{figure*}

\section{The Softmax Classifier}
Let $\vx \in \mathcal{X}$ be inputs and $\mathcal{Y}$ a set of $k$ labels. We define a softmax classifier as a model using functions $\vf:\mathcal{X}\to \RR^d$, $\vg:\mathcal{Y} \to \RR^d$ for $d>0$, assigns probabilities, $p$, to mutually exclusive labels, $y\in \mathcal{Y}$, given inputs, $\vx \in \mathcal{X}$, as:
\begin{align}
    \label{def:softmax_classifier}
    p(y\mid \vx) = \frac{\exp(\vf(\vx)^{\top}\vg(y))}{\sum_{y'\in \mathcal{Y}}\exp(\vf(\vx)^{\top}\vg(y'))} \; .    
\end{align}
So for each input the model makes a categorical distribution over the labels. The $\vf(\vx)$ are the representations of the inputs, and we call these the embeddings. The $\vg(y)$ are the representations of the labels and we call these the unembeddings. This paper considers the properties of the unembeddings only. This is not because properties of the embeddings are not interesting, but because we cannot talk about everything at once. The model class in \cref{def:softmax_classifier} is very broad including any image classifier using a softmax on the last logits, whether the encoder is a single Convolutional Neural Network \citep{lecun1989backpropagation}, a ResNet \citep{he2016deep} or a vision transformer \citep{dosovitskiy2020image}. The model class also includes large language models from early incarnations like BERT \citep{devlin2019bert}\footnote{See \citet[Appendix D.3]{roeder2021linear}.} and GPT-2~\citep{radford2019language} to more recent models like GPT-3~\citep{brown2020language}\footnote{See \citet[equation 2]{radford2018improving}} and gpt-oss-120B~\citep[see released code]{openai2025gptoss}. 

\section{Cosine of Unembeddings vs Probabilities}
\label{sec:cosine-vs-probs}

The first result of this paper relies on two properties. First, for a softmax classifier, adding some vector $\vv$ to all unembeddings, does not change the probability distributions of the model (this is not new, but we state it in \cref{lemma:Adding_vector_does_not_change_probs} and provide a full proof in \cref{app:adding_vector_does_not_change_probs} for completeness).

\begin{restatable}{lemma}{addingvectorsameprobs}
    \label{lemma:Adding_vector_does_not_change_probs}
    Let $\vv\in \RR^d$. For a softmax classifier as in \cref{def:softmax_classifier}, adding $\vv$ to all unembeddings, does not change the probability $p(y\mid \vx)$. 
\end{restatable}

Second, adding such a vector to the unembeddings can greatly change the cosine similarities. In the next two lemmas we see that for two vectors, we can always find a translation such that the translated vectors have cosine similarity $-1$ (\cref{lemma:adding_vector_cosine_m1}) or $1$ (\cref{lemma:adding_vector_cosine_1}). Proofs in \cref{app:adding_vector_change_cosine}.

\begin{restatable}{lemma}{cosineminusone}
    \label{lemma:adding_vector_cosine_m1}
    Let $\va, \vb \in \RR^d$ be vectors with $\va \neq \vb$. We can always choose $\vv \in \RR^d$ such that $\cos(\va +\vv,\vb + \vv) = -1$. 
\end{restatable}

\begin{restatable}{lemma}{cosineone}
    \label{lemma:adding_vector_cosine_1}
    Let $\va, \vb \in \RR^d$ be two vectors. We can always choose $\vv \in \RR^d$ such that $\cos(\va +\vv,\vb + \vv) = 1$. 
\end{restatable}

\subsection{Equivalent Models can have very Different Cosine Similarities Between Unembeddings}
\begin{theorem}
    \label{theorem:equiv_models_diff_cosine}
    Let $(\vf, \vg)$ be a softmax classifier as in \cref{def:softmax_classifier}. For two labels $y_i, y_j \in \mathcal{Y}$, assume the unembeddings of the labels are not equal, $\vg(y_i) \neq \vg(y_j)$. Let $\cos(\vg(y_i), \vg(y_j)) = c$ for some $c\in [-1, 1]$. We can then construct two equivalent models (all assigned probabilities will be the same for all inputs), $(\vf', \vg')$ and $(\vf'', \vg'')$, where for $(\vf', \vg')$ we have $\cos(\vg'(y_i), \vg'(y_j)) = -1$ and for $(\vf'', \vg'')$, we have $\cos(\vg''(y_i), \vg''(y_j)) = 1$. 
\end{theorem}
\begin{proof}
    By \cref{lemma:Adding_vector_does_not_change_probs}, we can make an equivalent model by adding the same vector, $\vv$ to all unembeddings. 
    For constructing the model $(\vf', \vg')$, we choose $\vv$ as in \cref{lemma:adding_vector_cosine_m1} and for $(\vf'', \vg'')$ we choose $\vv$ as in \cref{lemma:adding_vector_cosine_1}. 
\end{proof}

\cref{theorem:equiv_models_diff_cosine} says that the cosine similarity between two unembeddings does not tell us anything about the probabilities which the model assigns to the corresponding labels. For an example of three models which for all inputs assign the same probabilities, but where cosine similarity between their unembeddings vary greatly see \cref{fig:three_models_same_probabilities_diff_cosines} and details in the Appendix, \cref{app:examples}. Discussion of other measures in Appendix \cref{app:other_distance_measures}. \cref{lemma:adding_vector_cosine_m1,lemma:adding_vector_cosine_1} rely on translations of the unembeddings, therefore one might hope that centering the unembeddings (\ie making the mean across each feature dimension zero) could help. However, as we show in~\cref{app:example_centered,app:example_centered_length_1} in the Appendix, this still does not guarantee that high cosine similarity leads to high probabilities for the same inputs and fixing the length of the unembeddings is also not enough to get a guarantee.

\section{The Sigmoid Classifier}
\label{sec:sigmoid}
Let $\vx \in \mathcal{X}$ be inputs and $\mathcal{V}$ a set of $k$ labels.
In multi-label classification, our goal is to classify an input $\vx$ by assigning it a subset of non-mutually exclusive labels, $\vv \subseteq \mathcal{V}$.
We model such label assignments via indicator variables of the form $\mathbf{y} = (y_1, \ldots, y_k) \in \mathcal{Y}$ where $\mathcal{Y} = \set{+1, -1}^k$, and $y_v = +1$ denotes that label $v$ was assigned to $\vx$ while $y_v = -1$ denotes that it was not.
We define a sigmoid classifier which using functions $\vf:\mathcal{X}\to \RR^d$, $\vg:\mathcal{V} \to \RR^d$ for some $d>0$, assigns probabilities to each label $v\in \mathcal{V}$ given inputs, $\vx \in \mathcal{X}$ as:
\begin{align}
\label{eq:sigmoid_model}
p(y_v \mid \inputs) &= \sigmoid{\vg(v)\T \fembed}, \\
\sigmoid{z} &= \frac{1}{1 + e^{-z}}.
\end{align}
So for each input the model makes a Bernoulli distribution for each label.
The well known Word2Vec~\citep{mikolov2013} model using skip-gram with negative sampling can be interpreted as a sigmoid classifier (see \cref{app:Word2Vec_is_sigmoid} for details), where the goal is to classify which tokens are likely to co-occur together in a token window. Therefore for this model, the intuition that high cosine similarity connects to co-occurrence makes more sense.


\section{Cosine of Unembeddings vs Decision Regions}

\subsection{Cosine Similarities of Unembeddings Determine Sigmoid Decision Regions}

We let $\unembed \in \R^{k \times d}$ be the unembedding matrix with $i$'th row $\vg(y_i)\T$, and use $\vw_i$ as a shorthand. The decision region for a label combination $(y_1, y_2, \ldots , y_k) \in \set{+1, -1}^k$ is given by $\vf(\vx) \in \R^d$:
\begin{equation}
y_i\unembedrow{i}\T\vf(\vx) > 0,\quad \forall i: 1\leq i \leq k,
\end{equation}
see \cref{app:sigmoid_decision_regions_details} for details. Importantly, when $d<k$, there are label combinations that cannot be predicted~\citep{cover1965,grivas2024}.
 
\begin{definition}[Sign Pattern]
We determine the set of feasible label combinations as the sign pattern of $\unembed \in \R^{k \times d}$:
\begin{align}
\label{eq:sign_pattern_unembeds}
\signpattern{\unembed}= \set{\sign{\unembed \vz} : \vz \in \R^d,\, \unembedrow{i}\T\vz \neq 0, \, 1 \leq i \leq k},
\end{align}
where we disallow zeroes because we want to focus on the predicted label combinations and not parts of the decision boundaries\footnote{Allowing logits to be zero would complicate our set notation: our set would also count decision boundaries, when we are only interested in the set of outputs that can be predicted. Moreover, logits that are exactly zero are unlikely to occur, so ignoring the zero case is not an omission that matters in practice.}. The $\sign{}$ function is applied element-wise.
\end{definition}
Usually the $\vz$ in \cref{eq:sign_pattern_unembeds} would be our embeddings, $\vf(\vx)$. However, if $\vf$ is not surjective, we would get further restrictions which we will not go into here. When $d\geq k$ and $\rank\unembed \geq k$ all label combinations can be predicted and the sign pattern is not useful for telling such unembeddings apart. However, when $d < k$, there must exist label combinations that cannot be predicted and therefore the sign pattern can be used to distinguish between models. To focus on this more interesting case, imagine in the following that $d<k$ and $\rank\unembed \leq d$.
\textbf{Can we determine which label combinations are feasible from the cosine similarities of the unembeddings?}
We show that we can: we introduce the cosine similarity matrix $\mathbf{C}$ which we will show determines the sign pattern. 
 
\begin{definition}[Cosine Similarity Matrix]
Let $\mathbf{C} = \cosmat{\unembed} \in \R^{k \times k}$ be the symmetric matrix with entries $C_{i, j} = \cossim{\unembedrow{i}}{\unembedrow{j}}, \, i, j \in \set{1, \ldots, k}$,
where:
\begin{equation}
    \cossim{\unembedrow{i}}{\unembedrow{j}}  = \frac{\unembedrow{i} \T \unembedrow{j}}{\snorm{\unembedrow{i}} \snorm{\unembedrow{j}}}.
\end{equation}
\end{definition}

\begin{restatable}{lemma}{cossigmoid}
    \label{lemma:cos-sigmoid}
Let $\unembed \in \R^{k \times d}$ with non-zero rows. Then the cosine similarity matrix $\cosmat{\unembed}$ of the unembedding rows determines $\signpattern{\unembed}$.
\end{restatable}

Proof in \cref{app:lemma_cos-sigmoid}. This means that given the unembeddings and all cosine similarities between them, we know which combinations of labels can be predicted by our model.

\subsection{Cosine Similarities of Differences of Unembeddings Determine Softmax Decision Regions}
\label{sec:softmax_decision_boundaries}
For the softmax case, the decision region for label $t$ is given by feature vectors $\vf(\vx) \in \R^d$ that satisfy:
\begin{equation}
(\unembedrow{t} - \unembedrow{i})\T \vf(\vx) > 0,\quad \forall i: 1 \leq i \leq k,\quad i \neq t
\end{equation}
see \cref{app:softmax_decision_regions_details} for details. We note that in contrast to the sigmoid case, we now take pairwise differences of the rows of $\unembed$. We define the softmax decision regions by introducing the braid matrix~\citep[lecture 1]{stanley2004}, $\braid \in \R^{\binom{k}{2} \times k}$. This is the matrix whose rows are differences of unit vectors $\unitrow{i} - \unitrow{j},\, 1 \leq i < j \leq k$ in lexicographic order, see~\cref{app:softmax} for more details. Note that $\unembed' = \braid \unembed$ contains all pairwise differences of the rows of $\unembed$.
 
\begin{definition}[Ranking Pattern]
We define the set of feasible rankings for a classifier $\unembed$ as the ranking pattern:
\begin{align}
\rankingpattern{\unembed} = \begin{Bmatrix}\sign{\braid\unembed\vz}:\!\! & \vz \in \R^d, & \\ 
& (\unembedrow{i} - \unembedrow{j})\T\vz \neq 0,\!\!\!\!\! \\
& 1 \leq i < j \leq k\end{Bmatrix},
\end{align}
where as in the sign pattern case, we do not consider any points that fall on the decision boundary.
\end{definition}

The ranking pattern is important as it determines which rankings of labels our softmax classifier is able to produce or not.
If $\rank{\colcenter{\unembed}} = k-1$, where $\colcenter{\unembed}$ is the centered $\unembed$ with column means subtracted, then the model will be able to assign any ranking of the labels, \ie $\cardinality{\rankingpattern{\unembed}} = k!$. However, when $\rank{\colcenter{\unembed}} < k-1$, \eg when $d < k -1$, many rankings will be infeasible and therefore missing from the set~\citep{cover1967,grivas2022}.
\textbf{Can we determine which rankings are feasible from the matrix of cosine similarities?}

As we prove via counter-example in~\cref{app:counter-example}, this is not possible: while $\cosmat{\unembed}$ determines $\signpattern{\unembed}$, it does not determine $\rankingpattern{\unembed}$, and therefore neither does $\signpattern{\unembed}$, via~\cref{lemma:cos-sigmoid}. On the other hand, the cosine similarities of all the differences of unembeddings do determine the ranking pattern, see the proof in \cref{app:lemma_cos_softmax}.

\begin{restatable}{theorem}{counterexample}
\label{lemma:counterexample}
$\cosmat{\unembed}$ (and by extension, $\signpattern{\unembed}$) does not determine $\rankingpattern{\unembed}$.
\end{restatable}

\begin{restatable}{theorem}{cossoftmax}
\label{lemma:cos-softmax}
$\cosmat{\braid\unembed}$ determines $\rankingpattern{\unembed}$.    
\end{restatable}

\paragraph{Verification of Results}
While it is intractable to check the sign pattern and ranking pattern of large matrices due to the combinatorial explosion of outputs, we numerically checked~\cref{lemma:cos-sigmoid,lemma:cos-softmax} for unembedding matrices $\unembed$ with $n \in \set{3,5,10}$ and $d \in \set{2, \ldots \min{(6, n-1)}}$. We also provide a check for the counter-example of~\cref{lemma:counterexample}.
Our code implementing the checks is available on github.\footnote{\scriptsize{\url{https://github.com/bemigini/how-not-to-use-cosine-sim}}}

\section{Discussion and Further Questions}
\label{sec:discussion}

Our results highlight that the way we use cosine similarity cannot be independent of the model. For sigmoid classifiers, we can use the pairwise cosine similarity between all unembeddings to get which label combinations the model can predict (\cref{lemma:cos-sigmoid}). On the other hand, for a softmax classifier we need the pairwise cosine similarity between all \textit{differences} of unembeddings to determine which rankings the model can predict (\cref{lemma:cos-softmax}). We highlight that we only find which label combinations/rankings are possible, we do not know whether they will be assigned to any input, for this we need to consider the embeddings, $\vf(\vx)$.  

\paragraph{Unembeddings vs Embeddings.} Our result in \cref{theorem:equiv_models_diff_cosine} suggests that when working with softmax classifiers, we should not use cosine similarity between the unembeddings to say something about the assigned probabilities of the models. On the other hand, this result says nothing about whether cosine similarity is a meaningful measure to use between embeddings, as for example the probabilities of a softmax classifier are not invariant to vector translations of the input representations. 

\paragraph{Cosine Similarity in High Dimensional Spaces.} \citet{nielsen2025prediction} saw that when comparing token representations with each other it often resulted in nuisance hubness and therefore advocate for hubness reduction techniques. However, this work shows that the similarity measures themselves do not carry information about the probabilities, and this will probably not be fixed by doing hubness reduction. Related to this: In high dimensions for random vectors, cosine distance will concentrate. So if for trained models the cosine similarities vary a lot, this means that the unembeddings of trained models behave very differently from random matrices. What this ``different from random behaviour'' means is an interesting question which we hope will inspire future work.   

\paragraph{Using an Unembedding for Information Retrieval.} The fact that cosine similarity of unembeddings are not directly connected to probabilities, might also be part of the explanation for why the BERT [CLS] token is suboptimal for retrieval using cosine similarity, see e.g. table 1 in \citet{reimers2019sentence}.

\paragraph{Correlation of Cosine Distances.} In Section 4 of \citep{lee2025shared}, they show that some models have quite high correlations between their pairwise (cosine) distance matrices. But if cosine similarity between token unembeddings is not about probabilities, what do these high correlations tell us? Do they mean that there is a bias in the models for learning specific solutions? And if yes, where does such a bias come from?


\paragraph{Conclusion.} This paper shows examples of what cosine similarity can and cannot identify. First, for a softmax classifier high (or low) cosine similarity between unembeddings \textbf{cannot} be used to say that probabilities for the corresponding labels will be high (or low) at the same time. Second, for a sigmoid classifier all pairwise cosine similarities between unembeddings \textbf{can} identify which combinations of labels can be predicted. Third, for a softmax classifier, the pairwise cosine similarities \textbf{cannot} identify which rankings can be predicted, but the pairwise cosine similarities between all \textit{differences} of unembeddings \textbf{can}. Thus, before using cosine similarity one should consider whether it can answer the question of interest and which kind of classifier is relevant.

\section*{Acknowledgments}
B. M. G. N. was supported by the Novo Nordisk Foundation grant NNF24OC0092612.
A.G. was supported by ERC grant (``Numerical Analysis for Stable AI'', 101198795).

\bibliographystyle{plainnat}
\bibliography{sample}


\newpage
\appendix

\section{Proof that Adding Vector does not Change Probabilities}
\label{app:adding_vector_does_not_change_probs}

We recall \cref{lemma:Adding_vector_does_not_change_probs}:

\addingvectorsameprobs*

\begin{proof}
    Let $v\in \RR^d$. We will consider the calculation of probabilities for a model which has $\vg(y) + \vv$ as unembeddings for all $y\in \mathcal{Y}$ and show that this is the same as for the original model. Let $p'(y\mid \vx)$ denote the probability assigned by the model with $\vg(y) + \vv$ as unembeddings. We then have that
    \begin{align}
        p'(y\mid \vx) &= \frac{\exp(\vf(\vx)^{\top}(\vg(y)+\vv))}{\sum_{y'\in \mathcal{Y}}\exp(\vf(\vx)^{\top}(\vg(y')+\vv))} \\
        &= \frac{\exp(\vf(\vx)^{\top}\vg(y)+\vf(\vx)^{\top}\vv)}{\sum_{y'\in \mathcal{Y}}\exp(\vf(\vx)^{\top}\vg(y')+\vf(\vx)^{\top}\vv)} \\
        &= \frac{\exp(\vf(\vx)^{\top}\vg(y))\exp(\vf(\vx)^{\top}\vv)}{\sum_{y'\in \mathcal{Y}}\exp(\vf(\vx)^{\top}\vg(y'))\exp(\vf(\vx)^{\top}\vv)} \\
        &= \frac{\exp(\vf(\vx)^{\top}\vg(y))}{\sum_{y'\in \mathcal{Y}}\exp(\vf(\vx)^{\top}\vg(y'))} \\
        &= p(y\mid \vx).
    \end{align}
\end{proof}

\section{Proofs that Adding Vector can Change Cosine}
\label{app:adding_vector_change_cosine}

\subsection{Full Proof that Adding Vector can make Cosine -1}
\label{app:adding_vector_cosine_m1}
We recall \cref{lemma:adding_vector_cosine_m1}:

\cosineminusone*

\begin{proof}
    Let $\va, \vb \in \RR^d$ be two vectors such that $\va \neq \vb$. We first note that for any two vectors, $\vu, \vw \in \RR^d$, if $\vu + \vw = 0$, then $\cos(\vu, \vw) = -1$. This comes from the following:
    \begin{align}
        \label{eq:sum_zero_implies_cos_minus_1}
        \vu + \vw = 0 &\implies \vu = -\vw \\
        &\implies \cos(\vu, \vw) = \cos(-\vw, \vw) = -1 \;. \nonumber
    \end{align}
    We therefore choose $\vv = -(\va + \frac{1}{2}(\vb-\va))$ and see that
    \begin{align}
        (\va +\vv) + (\vb + \vv) &= \va - \va - \frac{1}{2}(\vb -\va) + \vb - \va - \frac{1}{2}(\vb-\va) \\
        &= \vb - \va - (\vb -\va) = 0 \;.
    \end{align}
    This means that $\cos(\va + \vv, \vb + \vv) = -1$ by \cref{eq:sum_zero_implies_cos_minus_1}.
\end{proof}

\subsection{Full Proof that Adding Vector can make Cosine 1}
\label{app:adding_vector_cosine_1}
We recall \cref{lemma:adding_vector_cosine_1}:

\cosineone*

\begin{proof}
    Let $\va, \vb \in \RR^d$ be two vectors. If $\va = \vb$, $\cos(\va, \vb) = 1$ and we are done. Therefore assume $\va \neq \vb$. In particular we have that $\vb - \va$ is not the zero vector. We now note that for any two vectors, $\vu, \vw \in \RR^d$, if $\vu - c\vw = 0$ for $c\in \RR_{>0}$, then $\cos(\vu, \vw) = 1$. This comes from the following:
    \begin{align}
        \label{eq:diff_zero_cos_1}
        \vu - c\vw = 0 &\implies \vu = c\vw \\
        &\implies \cos(\vu, \vw) = \cos(c\vw, \vw) \nonumber \\
        &= \frac{c \vw \cdot \vw}{\Vert c\vw\Vert \Vert \vw\Vert} = \frac{c \Vert \vw\Vert ^2}{c\Vert \vw\Vert \Vert \vw\Vert} = 1 \, . \nonumber
    \end{align}
    We therefore choose $\vv = -\va + \frac{1}{2}(\vb-\va))$ and see that
    \begin{align}
        \va + \vv = \va - \va + \frac{1}{2}(\vb - \va) = \frac{1}{2}(\vb-\va)
    \end{align}
    and 
    \begin{align}
        \vb + \vv = \vb - \va + \frac{1}{2}(\vb - \va) = \frac{3}{2}(\vb-\va) \,.
    \end{align}
    So we have that $(\vb + \vv) = 3(\va + \vv)$, which means that $\cos(\va + \vv, \vb + \vv) = 1$ by the observation in \cref{eq:diff_zero_cos_1}.  
\end{proof}

\section{Examples with Full Details}
\label{app:examples}

Code for the examples and plots can be found on github \\ \url{https://github.com/bemigini/how-not-to-use-cosine-sim}. 

\subsection{Example using Unrestricted Unembeddings}
\label{app:example_unrestricted}
We assume three models as in \cref{fig:three_models_same_probabilities_diff_cosines}, where all three models have embeddings as in \cref{fig:three_models_same_probabilities_diff_cosines} \textbf{a)}, and \textbf{b)}, \textbf{c)}, \textbf{d)} are the three different sets of unembeddings. The unembeddings of the first model corresponding to the five labels are in order 0 to 4: 
    \begin{align}
        \vl_0 &= \begin{pmatrix}
        1 \\
        0.5
    \end{pmatrix} \; , \; 
    \vl_1 = \begin{pmatrix}
        0.5 \\
        1
    \end{pmatrix}\; , 
    \vl_2 = \begin{pmatrix}
        -1 \\
        0.4
    \end{pmatrix}\; , \;
    \vl_3 = \begin{pmatrix}
        -0.8 \\
        -0.8
    \end{pmatrix}\; , \;
    \vl_4 = \begin{pmatrix}
        0.9 \\
        -1.2
    \end{pmatrix} \;. \nonumber
    \end{align} 
    We see that the cosine similarity between unembeddings corresponding to labels $0$ and $1$ is:
\begin{align}
    \cos(\vl_0, \vl_1) = \frac{1}{1.25} = 0.8 \;.
\end{align}
The unembeddings for the second model (\cref{fig:three_models_same_probabilities_diff_cosines} \textbf{c)}) are constructed by taking $\vv$ as in \cref{lemma:adding_vector_cosine_m1}, so we take 
\begin{align}
    \vv &= -\vl_0 - \frac{1}{2}(\vl_1 -\vl_0) \\
    &= -\begin{pmatrix}
        1 \\
        0.5
    \end{pmatrix} - \frac{1}{2} \left(\begin{pmatrix}
        0.5 \\
        1
    \end{pmatrix} - \begin{pmatrix}
        1 \\
        0.5
    \end{pmatrix} \right) = \begin{pmatrix}
        -0.75 \\
        -0.75
    \end{pmatrix} \;. \nonumber
\end{align}
The unembeddings of the second model are then $\vl_i + \vv$ and the cosine similarity between $\vl_0 + \vv$ and $\vl_1 + \vv$ is:
\begin{align}
    \cos(\vl_0 + \vv, \vl_1 + \vv) = -1 \; .
\end{align}
The unembeddings for the third model (\cref{fig:three_models_same_probabilities_diff_cosines} \textbf{d)}) are constructed by taking $\vv$ as in \cref{lemma:adding_vector_cosine_1}. That is, 
\begin{align}
    \vv = -\vl_0 + \frac{1}{2}(\vl_1 -\vl_0) \; ,
\end{align}
the unembeddings are $\vl_i + \vv$ and the cosine similarity between $\vl_0 + \vv$ and $\vl_1 + \vv$ is:
\begin{align}
    \cos(\vl_0 + \vv, \vl_1 + \vv) = 1 \; .
\end{align}

\subsection{Example using Centered Unembeddings}
\label{app:example_centered}

\begin{figure}[t]    
    \centering
    \begin{minipage}{0.49\textwidth}
        \centering
        \includegraphics[width = \textwidth]{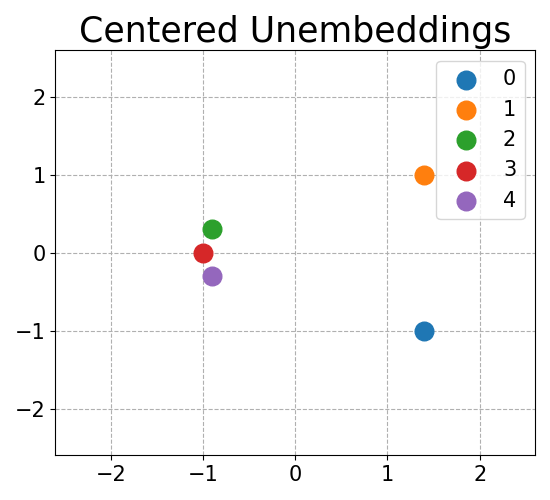}
    \end{minipage}%
    \begin{minipage}{0.49\textwidth}
        \centering
        \includegraphics[width = \textwidth]{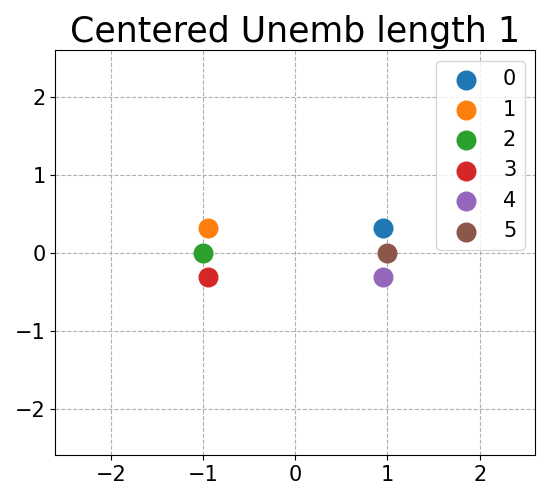}
    \end{minipage}%
    \caption{Examples of models with centered unembeddings. In both cases, we see that a high (or low) cosine similarity between unembeddings does not guarantee that the corresponding labels will have high (or low) probability for possible inputs. Left: These unembeddings are centered. The cosine similarity between $2$ and $4$ is much higher than between $2$ and $1$, but $2$ and $1$ can be likely at the same time while $2$ and $4$ cannot. Right: These unembeddings are centered and have length $1$. $1$ and $0$ can be likely at the same time while $1$ and $3$ cannot.}
    \label{fig:centered_unembeddings}
\end{figure}

We here present an example of a model with centered unembeddings illustrated in \cref{fig:centered_unembeddings} left. For the five labels the unembeddings are:
\begin{align}
    \vl_0 &= \begin{pmatrix}
        1.4 \\
        -1
    \end{pmatrix} \; , \; 
    \vl_1 = \begin{pmatrix}
        1.4 \\
        1
    \end{pmatrix}\; , 
    \vl_2 = \begin{pmatrix}
        -0.9 \\
        0.3
    \end{pmatrix}\; , \;
    \vl_3 = \begin{pmatrix}
        -1 \\
        0
    \end{pmatrix}\; , \;
    \vl_4 = \begin{pmatrix}
        -0.9 \\
        -0.3
    \end{pmatrix} \;.\nonumber
\end{align}
These unembeddings are centered since the sum is the zero vector. The cosine similarity between $\vl_0$ and $\vl_1$ is approximately $0.3$ and between $\vl_1$ and $\vl_2$ it is approximately $-0.6$, however for both these pairs of labels it is possible to place embeddings which will result in equal probability for both labels. On the other hand, the cosine similarity between $\vl_2$ and $\vl_4$ is $0.8$, but it is not possible for these labels to be tied for the highest probability, since if we place an embeddings between them, the dot-product with $\vl_3$ would be higher. 
This example also shows us that neighbourhoods using cosine similarity do also not correspond to probabilities: The two nearest neighbours of $\vl_2$ in terms of cosine similarity are $\vl_3$ and $\vl_4$, but the the two labels with which $\vl_2$ can have tied for highest probabilities are $\vl_3$ and $\vl_1$.

\subsection{Example using Centered Unembeddings of Length 1}
\label{app:example_centered_length_1}
We here present an example of a model with centered unembeddings which all have length $1$ illustrated in \cref{fig:centered_unembeddings} right. For the five labels the unembeddings are:
\begin{align}
    &\vl_0 = \begin{pmatrix}
        0.95 \\
        \sqrt{1-0.95^2}
    \end{pmatrix} \; , \; 
    \vl_1 = \begin{pmatrix}
        -0.95 \\
        \sqrt{1-0.95^2}
    \end{pmatrix}\; , \;
    \vl_2 = \begin{pmatrix}
        -1 \\
        0
    \end{pmatrix}\; , \\
    &\vl_3 = \begin{pmatrix}
        -0.95 \\
        \sqrt{1-0.95^2}
    \end{pmatrix}\; , \;
    \vl_4 = \begin{pmatrix}
        0.95 \\
        \sqrt{1-0.95^2}
    \end{pmatrix}\; , \;
    \vl_5 = \begin{pmatrix}
        1 \\
        0
    \end{pmatrix} \;.
\end{align}
We see that these sum to the zero vector and each of them has norm $1$. We see that the cosine similarity between $\vl_0$ and $\vl_1$ is approximately $-0.8$, while between $\vl_1$ and $\vl_3$ it is approximately $0.8$. At the same time, $\vl_0$ and $\vl_1$ can be tied for highest probability, while $\vl_3$ and $\vl_1$ cannot.

\section{Discussion of Other Distance Measures}
\label{app:other_distance_measures}

As discussed in the main paper, \cref{theorem:equiv_models_diff_cosine} tells us that cosine similarity between the unembeddings of a softmax classifier will not give us information about the probabilities.  \cref{theorem:equiv_models_diff_cosine} also means that the normalized Euclidean distance between unembeddings will also not tell us anything about the probabilities. The theorem also says that the dot product between unembeddings can have different signs for two equivalent models, and therefore the dot-product will also not give us information on the probabilities. To complete the list of commonly used distances: The Euclidean distance will also not give us information about the probabilities, since if we scale all unembeddings by $c > 0$ and all embeddings by $1/c$, then we do not change the probabilities, but we do change Euclidean distances.

\section{The Sigmoid Classifier Case}
\label{app:sigmoid_details}

In multi-label classification, our goal is to classify an input $\vx$ by assigning it a subset of non-mutually exclusive labels from a label vocabulary $\mathcal{V}$ of size $k$.
We treat each label $y_l,\, l \in \mathcal{V}$, as a Bernoulli random variable and we want to model label assignments of the form $\mathbf{y} = (y_1, \ldots, y_k) \in \mathcal{Y}$ where $\mathcal{Y} = \set{+1, -1}^k$ and a positive label $y_l = +1$ denotes that label $y_l$ was assigned to input $\vx$, while a negative label $y_l = -1$ denotes that $y_l$ was not.

\paragraph{Shared Feature Representation}
Perhaps the simplest model for multi-label classification is to fit an independent binary classifier per label $l \in \mathcal{V}$, parameterised by $\vg(l)$, on a shared embedding $\fembed$. We assign a probability to the presence/absence of each label as:
\begin{equation}
\label{eq:sigmoid-mlc}
p(y_l \mid \inputs) = \sigmoid{y_l\vg(l)\T \fembed},\quad \sigmoid{z} = \frac{1}{1 + e^{-z}},
\end{equation}
where we multiply the logits by $y_l$ to express both $p(y_l = +1 \mid \inputs)$ and $p(y_l = -1 \mid \inputs)$ in a single equation. This is possible because $\sigmoid{-z} = 1 - \sigmoid{z}$.

\subsection{Word2Vec as a Sigmoid Classifier}
\label{app:Word2Vec_is_sigmoid}

A lot of the language used around embeddings, such as claims that ``embeddings with high cosine similarity co-occur often in the data'' stem from foundational models of word embeddings, such as count-based co-occurrence models~\citep{deerwester1990} and early co-occurence based neural network models, such as Word2Vec~\citep{mikolov2013} and Glove~\citep{pennington2014}.
As we highlight in this paper, the co-occurrence rationale does not necessarily hold for softmax embeddings. The multi-label and multi-class setups differ in more than just task formulation: they have different implications for the geometric structure of the unembeddings. Below we connect the skip-gram model of Word2Vec with multi-label classification and sigmoid classifiers, to clarify that the co-occurrence framing is reasonable for sigmoid classifiers.

Word2Vec~\citep{mikolov2013} popularised neural network word embedding models.
The main idea is that words that appear in similar contexts should have similar representations.
This is formalised by picking a symmetric token window centred at a given token and learning a model that classifies whether two words co-occur in the token window or not.

Below we clarify that $\skipgram{\unembed}{\mathbf{Z}}$ with a label vocabulary of size $k$ and embedding dimensionality $d$, is effectively a multi-label classifier parametrised by $\unembed \in \R^{k \times d}$ together with a learnable set of input embeddings $\mathbf{Z} \in \R^{k \times d}$.
More specifically, we interpret the skip-gram model learned via negative sampling~\citep[Section 2.2]{mikolov2013} as multi-label classification,
where the positive labels come from the training data and the negative labels are sampled from a proposal distribution.
We now show that the negative sampling training objective corresponds to maximising the likelihood of incomplete label assignments. 
Our goal is to classify which tokens are likely to co-occur together in a token window and which ones are not.
We pick a token window size $w = 2r + 1,\, r \in \mathbb{N}^+$, consider the token window $C_t = (c_{t-r}, \ldots, c_t, \ldots, c_{t+r})$ centered around token $c_t$, and minimise the loss:
\begin{equation}
\label{eq:word2vecmlc}
    \mathcal{L}_t = -\sum_{-r \leq i \leq r,\, i \neq 0} \Bigl(  \log p(y_{c_{t + i}} = +1 \mid c_t)  + \sum_{j=1}^m \expectation[c_j \sim P_k(y)]{\log p(y_{c_j} = -1 \mid c_t)}\Bigr),
\end{equation}
where $c_t$ is the word index for the token at position $t$ in the text, $y_{c_i}$ is the binary label denoting whether token $c_i$ is in the window or not, and $m$ is the number of negative samples sampled from $P_k(c)$, which is a categorical distribution over the label vocabulary.\footnote{The paper uses $P_k(y) = U(k)^{\frac{3}{4}}/\mathcal{Z}$ where $U(k)$ is the unigram distribution computed over a corpus.}
By minimising $\mathcal{L}_t$ in~\cref{eq:word2vecmlc} we are maximising the likelihood of an incomplete label assignment: we have $2r$ active labels and $2rm$ inactive labels.\footnote{Strictly, $\mathcal{L}_t$ corresponds to a multi-label objective only when the tokens in $C_t$ combined with the negative samples are distinct.}.

The model is parametrised simply with embedding lookup tables, $\unembed, \mathbf{Z} \in \R^{k \times d}$ and the labels are scored via a sigmoid activation:
\begin{equation}
p(y_{c} \mid c') = \sigmoid{y_c\; \vg(c) \T \vf(c')}
\end{equation}
where $\vg(c)\T = \unembedrow{c}$ and $\vf(c')\T = \vz_{c'}$ are both embedding lookups, \ie rows, of $\unembed$ and $\mathbf{Z}$.


\paragraph{Decision Regions}
The decision regions given by $\unembed$ correspond to full-label assignments.
For $\mathcal{L}_t$ to be low, we would expect the decision regions for most contexts seen during training to exist.
As such, it makes sense that the decision regions of Word2Vec capture co-occurrence statistics and that this is reflected in the unembeddings.

\subsection{Decision Regions of a Sigmoid Multi-label Classifier}
\label{app:sigmoid_decision_regions_details}
We now analyse the decision regions for a multi-label classifier with unembedding matrix $\unembed \in \R^{k \times d}$. We first consider a single binary classifier.
The set of points in embedding space $\fembed \in \R^d$ that are assigned label $y$ are the points such that:
\begin{align}
\label{eq:mlc-decision-boundaries}
 p(y \mid \inputs) &> p(-y \mid \inputs) & \iff \\
 p(y \mid \inputs) &> 1 - p(y \mid \inputs) & \iff  \\
 p(y \mid \inputs) &> \frac{1}{2}  & \iff && \text{From \cref{eq:sigmoid-mlc}} \\
 \sigmoid{y \funembed \T \fembed} &> \frac{1}{2} & \iff && \text{(sigmoid is a bijection and $\sigmoid{0} = \frac{1}{2}$)} \\
 y \funembed \T \fembed &> 0 
\end{align}
We note that $\sigmoid{0} = \frac{1}{2}$.
The decision region for a label combination $(y_1, y_2, \ldots , y_k) \in \set{+1, -1}^n$ is given by $\vz \in \R^d$:
\begin{equation}
y_i\unembedrow{i}\T\vz > 0,\quad \forall i: 1 \leq i \leq k.
\end{equation}
$p(y_i = j ), j\in[-1, 1]$
We note that when $d < k$, there must exist label combinations that cannot be predicted. The label combinations that can be predicted by $\unembed$ are captured by the sign pattern of $\unembed$, which we define next.
 
\begin{definition}[Sign Pattern]
We determine the set of feasible label combinations as the sign pattern of $\unembed \in \R^{k \times d}$
is:
\begin{equation}
\label{app:eq_sign_pattern_unembeds}
\signpattern{\unembed}= \set{\sign{\unembed \vz} : \vz \in \R^d,\quad \unembedrow{i}\T\vz \neq 0 \quad 1 \leq i \leq k},
\end{equation}

where we disallow zeroes because we want to focus on the decision regions and not parts of the decision boundaries\footnote{Allowing logits to be zero would complicate our set notation: our set would also count decision boundaries, when we are only interested in the set of outputs that can be predicted. Moreover, logits that are exactly zero are unlikely to occur, so ignoring the zero case is not an omission that matters in practice.}, and where the $\sign{}$ function is applied element-wise:
\begin{equation}
\sign{z_i} = \begin{cases}
+ & \text{if } z_i > 0, \\
- & \text{if } z_i < 0 \\
0 & \text{otherwise}.
\end{cases}
\end{equation}
\end{definition}
We note that our sign pattern definition above extends naturally to matrices that have more columns than rank. If $\unembed' \in \R^{k \times d'}, \, d' > \rank{\unembed'} = d$, we can truncate the singular value decomposition of $\unembed' = \mathbf{U} \mathbf{\Sigma} \mathbf{V}\T$ to the top $d$ principal directions without losing any information, and use the resulting matrix instead.

\begin{proposition}
\textbf{$\signpattern{\unembed}$ is the set of orthants of $\R^k$ intersected by the span of the columns of $\unembed$}.
\label{prop:sign-colspace}
\end{proposition}
The above can be verified by noticing that $\mathbf{Wx}$ is a linear combination ($\mathbf{x}$ are the coefficients) of the columns of $\unembed$ and the sign function checks which orthant $\unembed\mathbf{x}$ falls in.

We can also think of these sign vectors in $\R^d$; they are the decision regions created by partitioning feature space with the decision boundaries of the binary classifiers $\set{\unembedrow{i}}_{i=1}^k$.
 
The sign pattern is useful as it tells us the set of possible label combinations that can be predicted via argmax prediction for a multi-label classifier parameterised by $\unembed$. We will now show that the cosine similarity matrix $\mathbf{C}$ determines $\signpattern{\unembed}$. So while we cannot recover the exact probability distributions we would obtain for an input $\inputs$, \textbf{we can recover the decision boundaries of the classifier} by knowing $\mathbf{C}$.
 
\begin{definition}[Cosine Similarity Matrix]
Let $\mathbf{C} = \cosmat{\unembed} \in \R^{k \times k}$ be the symmetric matrix with entries $C_{i, j} = \cossim{\unembedrow{i}}{\unembedrow{j}}, \, i, j \in \set{1, \ldots, k}$,
where:
\begin{equation}
    \cossim{\unembedrow{i}}{\unembedrow{j}}  = \frac{\unembedrow{i} \T \unembedrow{j}}{\snorm{\unembedrow{i}} \snorm{\unembedrow{j}}}.
\end{equation}
\end{definition}

\section{Pairwise Cosine Similarities Determine the Sign Pattern for a Sigmoid Classifier}
\label{app:lemma_cos-sigmoid}

We recall \cref{lemma:cos-sigmoid}

\cossigmoid*


\begin{proof}
Let us write $\cosmat{\unembed}$ as a product of matrices: 
\begin{equation}
\cosmat{\unembed} = \mathbf{DWW}\T\mathbf{D} = \mathbf{DGD},
\label{eq:cosmat}
\end{equation}
where
\begin{equation}
\mathbf{D} =
\begin{bmatrix}
1/\snorm{\unembedrow{1}} & 0 & \cdots & \\
0 & 1/\snorm{\unembedrow{2}} & & \\
 \vdots & & \ddots & \\
& & & 1/\snorm{\unembedrow{k}}
\end{bmatrix},
\end{equation}
and $\mathbf{G} \in \R^{k \times k}$ is the Gram matrix of $\mathbf{W}$, i.e.\ $G_{i, j} = \unembedrow{i}\T\unembedrow{j}$.
Our proof now proceeds in three steps by showing incrementally that the product of matrices does not change the sign pattern:
\begin{equation}
    \signpattern{\unembed} \overset{(1)}{=} \signpattern{\unembed \unembed\T} \overset{(2)}{=} \signpattern{\unembed \unembed\T \mathbf{D}} \overset{(3)}{=} \signpattern{\mathbf{D} \unembed \unembed\T \mathbf{D}},
\end{equation}
where the first equality uses the fact that $\unembed$ and $\unembed \unembed \T$ have the same column span, the second uses the fact that scaling the columns preserves the column span, and the third uses the fact that scaling the rows by positive scalars does not change the decision regions when we interpret the rows as normal vectors that define the decision boundary of binary classifiers.
\paragraph{1. $\signpattern{\unembed} = \signpattern{\mathbf{G}}$.}
We will use~\cref{prop:sign-colspace}, so it suffices to show that span of the columns for the two matrices is equal, \ie $\colspan{\unembed} = \colspan{\mathbf{G}}$.
This can be shown in two steps. \textbf{Step 1}: each column of $\mathbf{G}$ is a linear combination of the $d$ columns of $\mathbf{W}$, with coefficients given by each column of $\mathbf{W}\T$. More concretely, the $i^{th}$ column of $\mathbf{G}$ is:
\begin{equation}
\mathbf{G}_{(:, i)} = \unembed\unembedrow{i} = \left(\unembedrow{1}\T\unembedrow{i}, \cdots, \unembedrow{k}\T\unembedrow{i}\right)\T = \sum_{j=1}^d w_{ij} \underbrace{\left(w_{1j}, \cdots, w_{kj} \right)}_{\text{$j^{th}$ column of $\unembed$}}\T.
\end{equation}

Therefore, $\colspan{\unembed} \supseteq \colspan{\mathbf{G}}$. \textbf{Step 2:} To show that $\colspan{\unembed} = \colspan{\mathbf{G}}$, it remains to show that $\dim\colspan{\unembed} = \dim \colspan{\mathbf{G}}$.
We first need to show that $\ker{\unembed\T} = \ker{\unembed\unembed\T}$ so that our result will follow via rank nullity in~\cref{eq:rank-minus-ker}.
Forward direction: if $\vx \in  \ker \unembed\T $, then $\unembed\T \vx = \zeroes \implies \unembed\unembed\T\vx = \zeroes$, so $\vx \in \ker{\unembed\unembed\T} $.
Reverse direction: suppose $\vx \in \ker \unembed\unembed\T$, so $\unembed\unembed\T\vx = \zeroes$. Left multiplying by $\vx\T$, we have $ \vx\T \unembed\unembed\T\vx = 0 \implies (\unembed\T\vx)\T(\unembed\T\vx) = 0 \implies ||\unembed\T\vx||^2 = 0 \implies \unembed\T\vx = \zeroes$, so $\vx \in \ker \unembed\T$.

Via rank nullity and since $\unembed \in \R^{k \times d}$ has $k$ rows, we have:
\begin{align}
\label{eq:rank-nullity}
\rank{\unembed\T} + \dim{\ker \unembed\T} &= k \\ 
\label{eq:rank-nullity-2}
\rank{\unembed\unembed\T} + \dim \ker \unembed\unembed\T &= k.
\end{align}

By subtracting~\cref{eq:rank-nullity-2} from \cref{eq:rank-nullity} and rearranging terms, we get:
\begin{align}
\label{eq:rank-minus-ker}
\rank{\unembed\T} &= \rank{\unembed \unembed\T} + \dim \ker \unembed \unembed \T - \dim \ker \unembed \T  &\implies && \text{($\ker \unembed\unembed\T = \ker \unembed \T$)}\\
\rank{\unembed\T} &= \rank{\unembed \unembed\T} &\implies  && \text{($\rank {\unembed\T} = \rank{\unembed}$)} \\
\rank{\unembed} &= \rank{\unembed \unembed\T} 
\end{align}
Therefore, $\dim \colspan{\unembed} = \rank{\unembed} = \rank{\unembed \unembed\T} = \dim \colspan{\unembed \unembed \T} = \dim \colspan{\mathbf{G}}$.

From (1) and (2) we have $\colspan{\unembed} = \colspan{\mathbf{G}}$.
Because the sign pattern consists of the orthants intersected by the span of the columns~(\cref{prop:sign-colspace}) and $\colspan{\unembed} = \colspan{\mathbf{G}}$, we have $\signpattern{\unembed} = \signpattern{\mathbf{G}}$.
 
\paragraph{2. $\signpattern{\unembed} = \signpattern{\mathbf{GD}}$.} Scaling each column of $\mathbf{G}$ by a scalar does not change $\colspan{\mathbf{G}}$.
 
\paragraph{3. $\signpattern{\unembed} = \signpattern{\mathbf{DGD}}$.} Multiplying by $\mathbf{D}$ on the left does not change \texorpdfstring{$\signpattern{\mathbf{GD}}$}{S(W)}, despite the fact that \emph{it does} change the span of the columns.
This is because the decision boundary of a binary classifier does not change if we scale the normal vector which corresponds to the classifier by a positive scalar.
Sign vectors arise as decision boundaries of multiple such binary classifiers stacked in the rows of $\mathbf{GD}$\footnote{Another way of stating this result is: scaling the normal vectors of a hyperplane arrangement by positive coefficients does not change the arrangement, so the decision regions are preserved.}.

Via~\cref{eq:cosmat} we therefore have $\signpattern{\cosmat{\unembed}} = \signpattern{\mathbf{DGD}} = \signpattern{\unembed}$.
\end{proof}
 
\section{The Softmax Classifier Case}
\label{app:softmax}
In multi-class classification, our goal is to classify an input $\inputs$ from a set of mutually exclusive labels. We let $\vx \in \mathcal{X}$ be inputs and $\mathcal{Y} = \set{y_1, \ldots, y_k}$ be a set of $k$ labels. We model the probability for each class as:
\label{app:lemma_cos_softmax}
\begin{equation}
\label{eq:yi-softmax}
p(y_i \mid \inputs) = \softmax{\unembed \mathbf{x}}_i = \frac{\exp(\vf(\vx)^{\top}\vg(y_i))}{\sum_{j=1}^k\exp(\vf(\vx)^{\top}\vg(y_j))} \; . 
\end{equation}
\subsection{Decision Regions of a Softmax Multi-class Classifier}
\label{app:softmax_decision_regions_details}

For the softmax case, the decision region for label $y_t$ is given by embedding vectors $\fembed \in \R^d$ for which the probability of $y_t$ is larger than for $y_i$ and satisfy:
\begin{align}
\label{eq:mcc-decision-boundaries}
 p(y_t \mid \inputs) &> p(y_i \mid \inputs) \quad \forall i: 1 \leq i \leq k,\quad i \neq t & \iff && \text{(use~\cref{eq:yi-softmax})} \\
 \frac{\exp(\vf(\vx)^{\top}\vg(y_t))}{\sum_{j=1}^k\exp(\vf(\vx)^{\top}\vg(y_j))} &> \frac{\exp(\vf(\vx)^{\top}\vg(y_i))}{\sum_{j=1}^k\exp(\vf(\vx)^{\top}\vg(y_j))} & \iff  &&\text{(positive \& equal denominators)} \\
 \exp(\vf(\vx)^{\top}\vg(y_t)) &> \exp(\vf(\vx)^{\top}\vg(y_i)) & \iff &&\text{(log is strictly monotonic)} \\
 \vf(\vx)^{\top}\vg(y_t) &> \vf(\vx)^{\top}\vg(y_i) & \iff  \\
 \vf(\vx)^{\top} \bigl(\vg(y_t) -\vg(y_i)\bigr) & > 0 & \iff  \\
(\unembedrow{t} - \unembedrow{i})\T \vf(\vx) &> 0 \quad \forall i: 1 \leq i \leq k,\quad i \neq t.
\end{align}
We note that in contrast to the decision boundaries for the multi-label case from~\cref{eq:mlc-decision-boundaries}, we now take pairwise differences of the rows of $\unembed$.
We will now specify not only the $k-1$ differences between the target class and all other labels, but all pairwise differences, as this will allow us to reuse~\cref{lemma:cos-sigmoid}. Let $\braid \in \R^{\binom{k}{2} \times k}$ be the matrix whose rows are $\unitrow{i} - \unitrow{j},\, 1 \leq i < j \leq k$ in lexicographic order:
\begin{equation}
\braid = \begin{pmatrix} 1 & -1 & 0 & \cdots & 0 \\
1 & 0 & -1 & \cdots & 0 \\
\vdots & \vdots & \vdots & \ddots & \vdots \\
1 & 0 & 0 & \cdots & -1 \\
0 & 1 & -1 & \cdots & 0 \\
\vdots & \vdots & \vdots & \ddots & \vdots \\
0 & 0 & \cdots & 1 & -1 \end{pmatrix}.
\end{equation}
The rows of $\braid$ are the normal vectors of the Braid Hyperplane arrangement~\citep[Lecture 1]{stanley2004}, a well-studied object in combinatorics that splits $\R^k$ into $k!$ regions, each corresponding to a ranking of $k$ elements. We note that the set of rankings that can be produced by a low-rank classifier $\unembed$, is determined by the regions that are reachable in the image of $\unembed$, which is the span of its columns. We therefore define the set of feasible rankings for a classifier $\unembed$ as its ranking pattern.
\begin{definition}[Ranking Pattern]
We define the ranking pattern of $\unembed$ as:
\begin{equation}
\rankingpattern{\unembed} = \set{\sign{\braid\unembed\mathbf{z}}: \mathbf{z} \in \R^d ,\quad (\unembedrow{i} - \unembedrow{j})\T\mathbf{z} \neq 0, \quad  1 \leq i < j \leq k},
\end{equation}
where as in the sign pattern case, we do not consider any points that fall on the decision boundary.
\end{definition}
The ranking pattern is the set of orthants in $\R^{\binom{k}{2}}$ that can be intersected by $\colspan{\unembed'} = \colspan{\braid\unembed}$. Each sign vector in this set corresponds to a ranking over the $k$ labels via all pairwise comparisons.
As such, we have that $\rankingpattern{\eye{k}} = \set{\sign{\braid\inputs} : \inputs \in \R^k}$, gives us the set of all possible rankings of $k$ elements, where $\eye{k}$ is the $k \times k$ identity matrix.

Observe that $\unembed' = \braid \unembed$ contains all pairwise differences of the rows of $\unembed$, where each row of $\unembedb$ is a binary classifier $\unembedrow{}' = \unembedrow{i} - \unembedrow{j}$ which classifies an input as positive if it ranks $y_i > y_j$.
 
\textbf{The ranking pattern determines the softmax decision boundaries}, as they are a coarsening of the regions that correspond to rankings of the labels: the softmax decision region for class $t$ is the union of all ranking regions that rank class $t$ first~\citet[Section 3.3]{grivas2022}.
 
Observe that $\rankingpattern{\unembed} = \signpattern{\unembed'}$. We use this to derive the ranking pattern as the sign pattern of $\unembed'$ using our earlier argument.

\cossoftmax*

\begin{proof}
We apply~\cref{lemma:cos-sigmoid} for $\unembedb=\braid\unembed$, which gives us that $\cosmat{\braid\unembed}$ determines $\signpattern{\braid\unembed}$. We then note that $\signpattern{\braid\unembed} = \rankingpattern{\unembed}$.

We spell out more of the details below.
Define $\mathbf{G}' = \unembedb{\unembedb}\T = \braid\unembed\unembed\T\braid\T$. Now, similar to the sigmoid case, scale the matrix appropriately via $\mathbf{D}'$ to obtain the cosine similarity. Define $\mathbf{D}' \in \R^{\binom{k}{2} \times \binom{k}{2}}$ as
\begin{equation}
\mathbf{D}' =
\begin{bmatrix}
 1/\snorm{\unembedrow{1}-\unembedrow{2}} & 0 & \cdots & \\
 0 & 1/\snorm{\unembedrow{1}-\unembedrow{3}} & & \\
 \vdots & & \ddots & \\
& & & 1/\snorm{\unembedrow{k-1}-\unembedrow{k}}
\end{bmatrix}.
\end{equation}
Now we have:
\begin{equation}
\mathbf{C}' = \mathbf{D}'\unembedb{\unembedb}\T\mathbf{D}' = \mathbf{D'G'D'}
\end{equation}
Applying exactly the same arguments as for the sigmoid case above, we know that $\mathbf{C}'$ determines $\signpattern{\braid\unembed}$, and hence the ranking pattern $\rankingpattern{\unembed}$.
\end{proof}

\section{Two Counter-Examples}

\label{app:counter-example}
\subsection{Pairwise Cosine Similarities do not Determine the Ranking Pattern}

\counterexample*

\begin{proof}
    We prove this by giving an example of two models with the same cosine similarities between all their unembeddings, but where their possible rankings are different.

    Let $\unembed$ and $\unembed'$ be the unembedding matrices of two models where the unembeddings are placed at angles $0, \pi/4, \pi/2, 3\pi/4, \pi, 5\pi/4$ and $7\pi/4$. For $\unembed$ let the lengths of the unembeddings all be $5$ (see \cref{fig:cosine_same_ranks_different} left), and for $\unembed'$ let the lengths be $5$ for all but the unembedding for label $3$, which has length $2$ (see \cref{fig:cosine_same_ranks_different} right). 

    \begin{figure}[th]    
    \centering
    \begin{minipage}{0.49\textwidth}
        \centering
        \includegraphics[width = \textwidth]{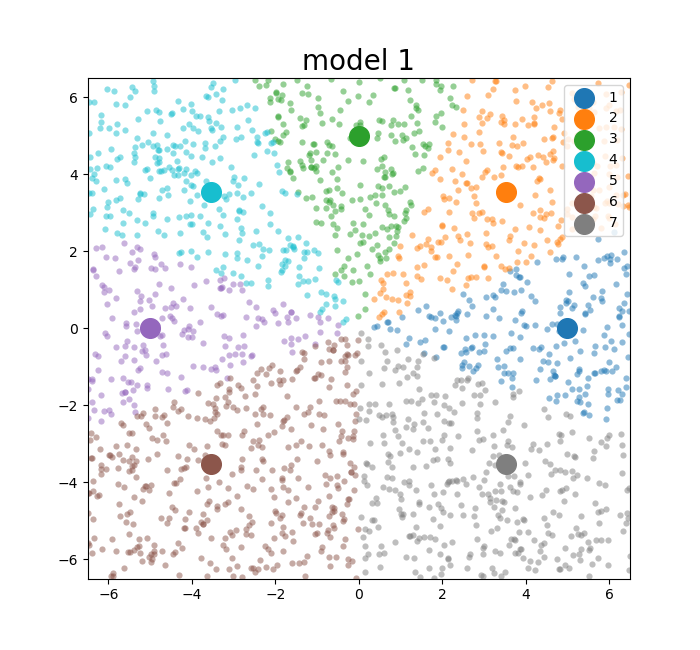}
    \end{minipage}%
    \begin{minipage}{0.49\textwidth}
        \centering
        \includegraphics[width = \textwidth]{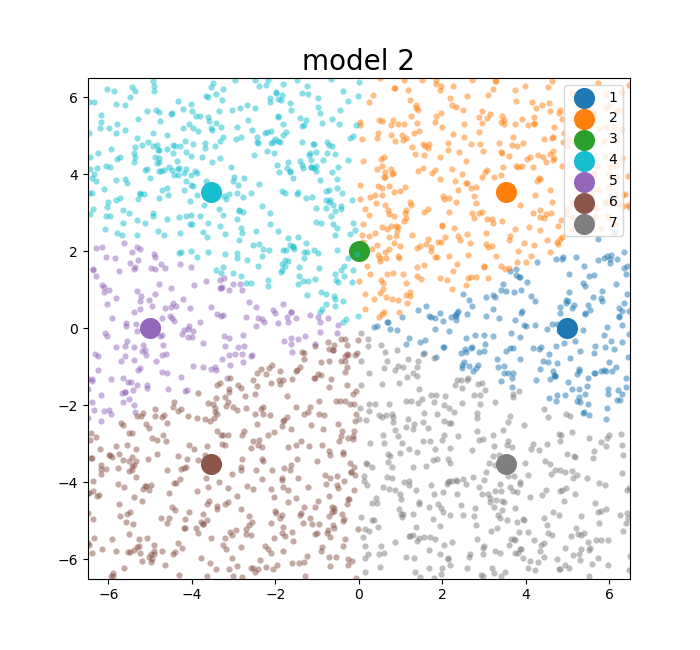}
    \end{minipage}%
    \caption{Example of two models with $7$ unembeddings placed at the same angles. Big dots are unembeddings. Small dots are embeddings coloured by their top $1$ label under the model. The model on the left has a region where label $3$ is the top $1$ label, but the model on the right has no region where label $3$ is the most likely. In other words the ranking pattern of model $2$ does not include any rankings where label $3$ is the top $1$. }
    \label{fig:cosine_same_ranks_different}
    \end{figure}
    
The ranking patterns of these two models are different, since model 1 has a region where label $3$ is the top $1$ label, but the model on the right has no region where label $3$ is the most likely. In other words the ranking pattern of model $2$ does not include any rankings where label $3$ is the top $1$, while the ranking pattern of model $1$ includes rankings of this type.

This example proves that $\cosmat{\unembed}$ does not determine $\rankingpattern{\unembed}$, and by extension via~\cref{lemma:cos-sigmoid}, $\signpattern{\unembed}$ does not determine $\rankingpattern{\unembed}$.
\end{proof}

However, given that our counter-example relies on scaling one of the vectors, which works for vector configurations, one may ask whether this still holds for affine configurations. In the next section we show by counter-example that even for the affine case $\signpattern{\unembed}$ does not determine $\rankingpattern{\unembed}$.

\subsection{The Sign Pattern does not Determine the Ranking Pattern}

\begin{restatable}{theorem}{counterexample2}
\label{lemma:counterexample2}
$\signpattern{\unembed}$ does not determine $\rankingpattern{\unembed}$.
\end{restatable}

The sign pattern of our vector configuration in the previous example was invariant to scaling each row by a positive scalar, while the ranking pattern was not.
We can move from a vector configuration in $\R^d$ to a point configuration in $\R^{d+1}$ by including a constant column. In our example unembedding matrices in~\cref{fig:point-config} below we assume that the first column is constant, as can be seen by the column of ones.

If the sign pattern of the point configuration determined the ranking pattern, then any pair of such unembedding matrices having the same sign pattern would have the same ranking pattern.
We show this is not true by providing a counter-example from~\citep[Section 1.10, Figures 1.10.1 \& 1.10.2]{OrientedMatroids1999} of two unembedding matrices $\unembed_1, \unembed_2 \in \R^{5 \times 3}$, for which $\signpattern{\unembed_1} = \signpattern{\unembed_2}$ but $\rankingpattern{\unembed_1} \neq \rankingpattern{\unembed_2}$.

\paragraph{Sign Pattern}
We think of the sign pattern of a point configuration $\unembed \in \R^{k \times (d+1)}$ as follows. When classifying a feature vector $\vx$ via $\sign{\unembed\vx}$, because the first column is constant, we can rewrite the computation as $\unembed\vx = \unembed_{(:, 1:)}\vx_{1:} + \ones x_1$, where we are rewriting the matrix-vector multiplication as a linear combination of the columns of $\unembed$ and using the fact that the first column is a column of ones. Now, from this perspective, we can think of $\vx_{1:} \in \R^d$ as a normal vector of an affine hyperplane which is offset by $x_1$. The sign pattern is then given by thinking of each row of $\unembed_{(:, 1:)}$ as a point in $\R^{d}$ and checking what side of the affine hyperplane each point can fall on, for all possible affine hyperplanes.



\begin{figure}
    \centering
\begin{tikzpicture}[font=\scriptsize, scale=1.1]
    \def\rows{%
    {1.00, -0.80, 1.00},%
    {1.00, 0.80, 1.00},%
    {1.00, 0.55, -0.50},%
    {1.00, -1.00, -1.00},%
    {1.00, 1.00, -1.00}}

    \loadrows{\rows}
    \placecoords{5}
    
    \begin{scope}[shift={(-4.0,0)}]
        \drawmatrixV{\unembed_1 \in \mathbb{R}^{5 \times 3}}
    \end{scope}
    
    \drawpoints{5}
\end{tikzpicture}
\hfill
\begin{tikzpicture}[font=\scriptsize, scale=1.1]
    \def\rows{%
    {1.00, -1.00, 1.00},%
    {1.00, 1.00, 1.00},%
    {1.00, 0.35, -0.50},%
    {1.00, -0.80, -1.00},%
    {1.00, 0.80, -1.00}}

    \loadrows{\rows}
    \placecoords{5}
    
    \begin{scope}[shift={(-4.0,0)}]
        \drawmatrixV{\unembed_2 \in \mathbb{R}^{5 \times 3}}
    \end{scope}
    
    \drawpoints{5}
\end{tikzpicture}
\\
\vspace{.5cm}
\begin{tikzpicture}[scale=.9]
\begin{scope}[shift={(0.0,0.0)}]
    \tikzset{braidline/.style={gray, line width=.4pt}}

    \draw[braidline] (0.000,2.500) -- (-0.000,-2.500);  
    \draw[braidline] (1.858,1.672) -- (-1.858,-1.672);  
    \draw[braidline] (2.488,-0.249) -- (-2.488,0.249);  
    \draw[braidline] (1.858,1.672) -- (-1.858,-1.672);  
    \draw[braidline] (2.466,-0.411) -- (-2.466,0.411);  
    \draw[braidline] (1.858,-1.672) -- (-1.858,1.672);  
    \draw[braidline] (2.488,0.249) -- (-2.488,-0.249);  
    \draw[braidline] (0.768,-2.379) -- (-0.768,2.379);  
    \draw[braidline] (1.858,1.672) -- (-1.858,-1.672);  
    \draw[braidline] (0.000,2.500) -- (-0.000,-2.500);  

    \node[circle, fill=black, inner sep=0pt, minimum size=1.2mm] at (0,0) {};

    \contourlength{0.01pt}
    \begin{scope}[opacity=.4]
        \rv{25314}{2.469}{1.092}{r25314}{}
        \rv{21354}{1.098}{2.466}{r21354}{}
        \rv{12345}{-0.420}{2.667}{r12345}{}
        \rv{12435}{-1.468}{2.266}{r12435}{}
        \rv{14235}{-2.432}{1.172}{r14235}{}
        \rv{14325}{-3.668}{0.488}{r14325}{}
        \rv{41352}{-2.469}{-1.092}{r41352}{}
        \rv{45312}{-1.098}{-2.466}{r45312}{}
        \rv{54321}{0.420}{-2.667}{r54321}{}
        \rv{53421}{1.468}{-2.266}{r53421}{}
        \rv{53241}{2.432}{-1.172}{r53241}{}
        \rv{52341}{3.668}{-0.488}{r52341}{}
    \end{scope}
        \rv{41325}{-2.800}{0.000}{r41325}{}
        \rv{52314}{2.800}{0.000}{r52314}{}
\end{scope}
\end{tikzpicture}
\hfill
\begin{tikzpicture}[scale=.9]
\begin{scope}[shift={(0.0,0.0)}]
    \tikzset{braidline/.style={gray, line width=.4pt}}

    \draw[braidline] (0.000,2.500) -- (-0.000,-2.500);  
    \draw[braidline] (1.858,1.672) -- (-1.858,-1.672);  
    \draw[braidline] (2.488,0.249) -- (-2.488,-0.249);  
    \draw[braidline] (1.858,1.672) -- (-1.858,-1.672);  
    \draw[braidline] (2.294,-0.994) -- (-2.294,0.994);  
    \draw[braidline] (1.858,-1.672) -- (-1.858,1.672);  
    \draw[braidline] (2.488,-0.249) -- (-2.488,0.249);  
    \draw[braidline] (0.997,-2.293) -- (-0.997,2.293);  
    \draw[braidline] (1.858,1.672) -- (-1.858,-1.672);  
    \draw[braidline] (0.000,2.500) -- (-0.000,-2.500);  

    \node[circle, fill=black, inner sep=0pt, minimum size=1.2mm] at (0,0) {};

    \contourlength{0.01pt}
    \begin{scope}[opacity=.4]
    \rv{25314}{2.469}{1.092}{r25314}{}
    \rv{21354}{1.098}{2.466}{r21354}{}
    \rv{12345}{-0.550}{2.643}{r12345}{}
    \rv{12435}{-1.578}{2.191}{r12435}{}
    \rv{14235}{-2.272}{1.459}{r14235}{}
    \rv{14325}{-2.613}{0.679}{r14325}{}
    \rv{41352}{-2.469}{-1.092}{r41352}{}
    \rv{45312}{-1.098}{-2.466}{r45312}{}
    \rv{54321}{0.550}{-2.643}{r54321}{}
    \rv{53421}{1.578}{-2.191}{r53421}{}
    \rv{53241}{2.272}{-1.459}{r53241}{}
    \rv{52341}{2.613}{-0.679}{r52341}{}
    \end{scope}
    \rv{25341}{2.800}{0.000}{r25341}{}
    \rv{14352}{-2.800}{0.000}{r14352}{}
\end{scope}
\end{tikzpicture}
    \caption{Two point configurations $\unembed_1, \unembed_2 \in \R^{5 \times 3}$ that have the same sign pattern but different ranking pattern. We illustrate the two matrices and the corresponding points on the right: each row corresponds to a point of the same colour and has coordinates given by the second and third column. The sign pattern arises by taking all possible affine hyperplanes and separating out the points. Below each point configuration is the partition of the feature space in $\R^2$ into regions that correspond to which ranking would be predicted, where we fade out rankings that appear in both configurations. As can be seen, \trv{41325} and \trv{52314} only appear on the left and  \trv{14352} and \trv{25341} only appear on the right. Therefore, although the sign pattern is the same, the ranking pattern differs, proving that the sign pattern cannot determine the ranking pattern.}
    \label{fig:point-config}
\end{figure}

In~\cref{fig:point-config}, we show two unembedding matrices, $\unembed_1, \unembed_2 \in \R^{5 \times 3}$ that have the same sign pattern but different ranking patterns.
To see the sign pattern, we classify the points into positive and negative with any possible affine hyperplane in $\R^2$. For example, if take an affine hyperplane that leaves all points on the same side, we can produce \tsv{+++++}, and \tsv{-----} by flipping the hyperplane. Since the point corresponding to the first row is linearly separable from the rest, we can also produce \tsv{+----} and \tsv{-++++}. Continuing in this way, the sign pattern for both unembedding matrices is:
\begin{align*}
\signpattern{\unembed_1} = \signpattern{\unembed_2} \supset \bigl\{
&\tsv{+++++}, && \text{(All points on one side)} \\ 
&\tsv{+----}, \tsv{+-+++}, \tsv{++++-}, \tsv{+++-+},&& \text{(All trapezoid vertices, so $\textcolor{clr3}{\unembedrow{3}}$ is missing)} \\
&\tsv{++---}, \tsv{+-++-}, \tsv{+++--}, \tsv{+--+-}, && \text{(All trapezoid edges)}\\
&\tsv{++-+-}\bigr\}, && \text{(Remaining)}
\end{align*}
where for simplicity we show half the sign vectors, \ie only those that start with a \tsv{+}. The remaining negated sign vectors complete the sign pattern.

\paragraph{Ranking Pattern}
For the ranking pattern, when we compute $\braid\unembed$, we note that the constant column cancels to $0$ for the pairwise differences of the rows, and as such we can also think of the feasible rankings in $\R^d$ by ignoring the first column. We show the way feature space is partitioned into rankings at the bottom of~\cref{fig:point-config}.
As can be seen, the ranking pattern for $\unembed_1$ and $\unembed_2$ differ: the rankings \trv{41325} and \trv{52314} are feasible for $\unembed_1$ but not for $\unembed_2$, while the rankings \trv{14352} and \trv{25341} are feasible for $\unembed_2$ but not for $\unembed_1$.

\end{document}